\documentclass[letterpaper, 10 pt, conference]{ieeeconf}  

\IEEEoverridecommandlockouts                              

\usepackage{graphicx}      
\usepackage{comment}
\usepackage{enumerate}
\usepackage[utf8]{inputenc} 
\usepackage{algorithmic}
\usepackage{graphicx}
\usepackage{subfig}
\usepackage{textcomp}
\usepackage[table]{xcolor}
\usepackage{xcolor}
\usepackage{cite}
\usepackage{csquotes}
\usepackage{multicol}
\usepackage{tikz}
\usepackage{hyperref}
\usepackage{siunitx}
\usepackage{array} 
\usepackage{hyperref} 
\hypersetup{backref=true,       
    pagebackref=true,               
    hyperindex=true,                
    colorlinks=true,                
    breaklinks=true,                
    urlcolor= black,                
    linkcolor= blue,                
    bookmarks=true,                 
    bookmarksopen=false,
    filecolor=black,
    citecolor=blue,
    linkbordercolor=blue
}

\usepackage{mathtools}
\usepackage{tcolorbox}
\usepackage{multirow}

\newtcbox{\mybox}[1]{nobeforeafter, colframe=black, colback=white, boxrule=0.5mm, width=\linewidth, arc=0mm, boxsep=0mm, left=5mm, right=5mm}

\usepackage{balance}
\usepackage{soul} 

\usepackage{tikz}

\newtheorem{proposition}{Proposition}[section]
\newcommand\numeq[1]%
  {\stackrel{\scriptscriptstyle(\mkern-1.5mu#1\mkern-1.5mu)}{=}}
\newtheorem{lemma}{Lemma}[section]

\newtheorem{remark}{Remark}[section]

\newtheorem{assumption}{Assumption}
\newtheorem{problem}{Problem}
\usepackage{hyperref}
\usepackage{times} 
\usepackage{amsmath} 
\usepackage{amssymb}  

\newcommand\Moesays[1]{\textcolor{black}{#1}} 

\usepackage{soul}

\usepackage{eso-pic}
\usepackage{graphicx}

\title{\LARGE \bf
\Moesays{\textcolor{black}{Adaptive} Non-linear Centroidal MPC with Stability Guarantees \\ for Robust Locomotion of Legged Robots} 
}

\author{Mohamed Elobaid$^1$, Giulio Turrisi$^2$, Lorenzo Rapetti$^1$, Giulio Romualdi$^1$, Stefano Dafarra$^{1}$, \\  Tomohiro Kawakami$^{3}$,  Tomohiro Chaki$^{3}$, Takahide Yoshiike$^{3}$,  Claudio Semini$^2$, and Daniele Pucci$^{1,4}$
\thanks{$^{1}$ Artificial and Mechanical Intelligence (AMI), Istituto Italiano di Tecnologia (IIT); Genoa, Italy {\tt\small {\{firstname.lastname\}@iit.it}}.}%
\thanks{$^{2}$ Dynamic Legged Systems (DLS), Istituto Italiano di Tecnologia (IIT); Genoa, Italy {\tt\small {\{firstname.lastname\}@iit.it}}.}%
\thanks{$^{3}$ Frontier Robotics, Innovative Research Excellence; Honda R\&D,  Saitama, Japan {\tt\small {\{firstname.lastname\}@jp.honda}}}
\thanks{$^{4}$ Machine Learning and Optimisation, The University of Manchester, Manchester, United Kingdom.}%
}

\AddToShipoutPictureBG*{%
	\AtPageUpperLeft{%
		\setlength{\unitlength}{1mm}%
		\put(0,-12){\makebox(\paperwidth,0)[c]{\parbox{0.8\textwidth}{\centering\textcolor{gray}{\large This paper has been accepted for publication with the IEEE Robotics and Automation Letters, 2025.  ©IEEE}}}}
	}
}

\begin{document}

\maketitle

\thispagestyle{empty}

\pagestyle{empty}

\begin{abstract}     
Nonlinear model predictive locomotion controllers based on the reduced centroidal dynamics are nowadays ubiquitous in legged robots. These schemes, even if they assume an inherent simplification of the robot's dynamics, were shown to endow robots with a step-adjustment capability in reaction to small pushes, and in the case of uncertain parameters - as unknown payloads - they were shown to provide some \enquote{practical}, albeit limited,  robustness. In this work, we provide rigorous certificates of their closed-loop stability via reformulating the \textcolor{black}{online} centroidal MPC controller. This is achieved thanks to a systematic procedure inspired by the machinery of adaptive control, together with ideas coming from Control Lyapunov Functions. Our reformulation, in addition, provides robustness for a class of unmeasured constant disturbances. To demonstrate the generality of our approach, we validated our formulation on a new generation of humanoid robots - the $\SI{56.7}{kg}$ ergoCub, as well as on the commercially available $\SI{21}{kg}$ quadruped robot Aliengo.
\end{abstract}

\smallskip 

\section{Introduction}

Legged robots are attracting considerable interest both from researchers and industry practitioners. New companies and existing industry leaders are joining the race for a \enquote{\textit{general purpose}} robot. This general-purpose robot is expected to perform collaborative and autonomous tasks both in home and work environments \cite{lorenzoICRA2023, turrisi2024collaborative}. This, in turn, requires navigating those environments safely while carrying out collaborative tasks robustly. This paper contributes towards the design of locomotion controllers that ensure a degree of provable robustness and stability.

On the locomotion side, a commonly employed modular architecture that separates the trajectory generation and adjustment blocks, typically running at a \enquote{slower} frequency, and a faster whole-body trajectory tracking block, is well-trialed \cite{Koolen, BenchmarkingGiulio}. The trajectory generation and adjustment blocks usually use simplified \enquote{template} models to allow faster computation, including the Linear Inverted Pendulum \cite{englsberger,griffin_ihmc, benchmarking_laas}, and the centroidal momentum dynamics model \cite{Orin}, which have demonstrated effectiveness and are well documented in the field. Concerning the latter, a common approach for quadruped locomotion is the use of the single-rigid-body model that captures the centroidal dynamics in addition to the robot base orientation evolution and inertial effects \cite{quadruped_1, turrisi2024sampling}.  On the humanoids side, recently a centroidal Model Predictive Controller (MPC) was utilized, in this modular architecture, to allow a humanoid robot to walk \cite{GiulioICRA}. Moreover, by modifying the prediction model and estimating contact wrenches at the hand (under suitable assumptions and contact wrenches parametrization) the robot was shown to walk under the action of a persistent disturbance, namely while carrying a payload \cite{ElobaidICRA}. However, no closed-loop stability nor robustness guarantees were provided from a methodological perspective.


Considerable work is being done to endow the envisioned general-purpose robots with robustness to both impulsive disturbances (e.g. in case of push recovery \cite{push_recovery}) and persistent disturbances (e.g. in case of payload carrying \cite{Harada_1, Harada_2, Kheddar}) while performing collaborative tasks. These requirements naturally lead to the question of endowing the locomotion controller with an \textit{adaptive-control} flavour \cite{adaptive_control_book, Isidori_adaptive, slotine_li}. 
From this adaptive control point of view, the key lies in understanding that the centroidal momentum dynamics, when influenced by unmeasured disturbances, falls within the class of systems known as \textit{parametric-pure-feedback forms}. This fact in turn allows us to leverage the procedure described in the seminal work of Kanellakopoulos, Kokotovic and Morse  \cite{pure_parametric_feedback_form}, extending it to the Multi-Input-Multi-Output case, for designing an adaptation scheme and feedback.

Furthermore, to handle constraints typically imposed at the trajectory adjustment layer on the control and states (e.g. \textcolor{black}{friction cone constraints}), and inspired by the ideas in \cite{clf_qp, clf_nmpc, adaptive_clf_nmpc}, we combine the Lyapunov stability machinery synthesized via the adaptive control approach with the nonlinear Centroidal MPC.  Additionally, we add constraints on the residual dynamics, which in our setting reduces to the angular momentum evolution, thus ensuring the stability of the whole feedback system. 

Compared to the standard literature on stability for MPC controllers (see e.g. \cite[Ch~5]{grune}), the proposed stabilizing constraints are obtained without relying on computing maximal controlled invariant output reachable sets. Additionally, the developed stabilizing constraints rely on non-restrictive assumptions and provide guarantees of robustness to bounded constant disturbances as a by-product thanks to a simple adaptation scheme. Moreover, differently from similar important works on intrinsically stable MPC for humanoid gait generation e.g. \textcolor{black}{\cite{oriolo_tro, smaldoneHumanoids}}, we leverage a nonlinear reduced model for prediction which \textcolor{black}{facilitates explicitly constraining contact forces within the friction cone, allowing to handle non-coplanar contacts \cite{mpc_legged_survey}, performing more agile motions given the more accurate dynamics representation, as well as handling surfaces with different static friction coefficients \cite{RF-MPC}}. Finally, in contrast to similar works on quadrupeds, we derive an adaptation law and stabilizing constraints \textcolor{black}{without simplifying to a \textit{linear} force-based  MPC \cite{mohsen}, nor} relying on the Slotine-Li formulation \cite{slotine_li, hutter}, \textcolor{black}{neither requiring multi-stage optimization as done in \cite{xu}}, thereby \textcolor{black}{considerably streamlining} the controller
synthesis. Consequently, the contributions of this paper are twofold; 
\begin{enumerate}[(i)]
    \item A reformulation of \textcolor{black}{an online} Centroidal MPC with added stability and robustness guarantees is presented. 
    \item To prove the generality of our approach, we experimentally validate the proposed controller in different scenarios under multiple types of disturbances, both on a humanoid and a quadruped robot. Moreover, the code for reproducing the experiments is made open source\footnote{\url{https://github.com/ami-iit/paper_elobaid_2024_stable-centroidal-mpc}}. 
\end{enumerate} 

This paper is organized as follows: Section \ref{sec:prel} presents notations and the necessary machinery used throughout the rest of the manuscript, together with a formal statement of the problem; Section \ref{Sec:three} introduces the proposed reformulation of the Centroidal MPC and states the main results; Section \ref{sec:validation} introduces, in a brief manner, the experimental setup and presents comprehensive experimental validation results. Finally, concluding remarks in Section \ref{sec:conclusion} end the manuscript.

\section{Background}\label{sec:prel}

\subsection{Notation and nomenclature}  Given a vector $x \in \mathbb{R}^n$, $\|x\|$ and $x^\top$ define, respectively, the $\ell_2$ norm and transpose of $x$. \textcolor{black}{Matrices are denoted with capital letters e.g. $M \in \mathbb{R}^{d_1\times d_2}$}. $I_n$ and $0_n$ denote the identity and zero matrices of dimension $n$. For $x \in \mathbb R^3$, $x^\wedge = S(x) : \mathbb R^3 \to \mathfrak{so(3)}$ returns the skew symmetric matrix form of $x$.  
Additionally, the following nomenclature is used;
\begin{itemize}
    \item $h^{\ell}, \ h^{\omega}$ denote the \textit{aggregate} linear and angular momentum of all links referred to the robot center of mass - CoM, and oriented as the inertial frame and $h$ is the vector collecting them.
    \item $p_{CoM}$ is the CoM position referred to the inertial frame.
    \item $p_i$ represents the position of a contact point associated with contact force $i$ referred to the inertial frame.
    \item $m$ is the robot mass \textcolor{black}{and $\ \Vec{g} = \begin{pmatrix}
        0 & 0 & g & 0 & 0 & 0
    \end{pmatrix}^\top$ denotes the gravity acceleration vector with $g = -9.81$}.
\end{itemize}

\subsection{The centroidal momentum dynamics}

The \textit{reduced} dynamics of the centroidal quantities of a rigid body in contact with the environment \textcolor{black}{with $n_c$ control impact forces}, under the influence of $k$ external disturbance forces \textcolor{black}{acting on the CoM} takes the form \cite{Nava}:
\begin{subequations}\label{momentumdyn_ct}
\begin{align}
&\dot{p}_{CoM} = \frac{1}{m}Bh \label{com_dyn}\\
&\dot h = \textcolor{black}{\sum_{i = 1}^{n_c} A_{u_i}(p_{u_i})\Gamma_i u_i + \sum_{i = 1}^{k} A_{\theta_i}(p_{\theta_i})\theta_i}  + m\Vec{g} \label{momentum_dyn} 
\end{align}
\end{subequations}
\textcolor{black}{where $u_i, \ \theta_i \in \mathbb{R}^3$ are control impact forces and external disturbance forces respectively, $\Gamma_i \in \{1, 0\}$ is a variable capturing the status of contact $i$ (see eqn (5) in \cite{GiulioICRA}),  $B = \begin{bmatrix} I_{3} & 0_{3}\end{bmatrix}$ a selector matrix, and for some force $\gamma$}:
\begin{align*}
\textcolor{black}{A_\gamma(p_\gamma)} &= \textcolor{black}{\begin{bmatrix}
     I_{3} \\ S(p_\gamma - p_{CoM}) 
\end{bmatrix}}
\end{align*}
 Whenever mentioned, the nominal (unperturbed) dynamics, is obtained by dropping the terms $A_{\theta_i}(p_{\text{CoM}}, p_{\theta_i})\theta_i$ from (\ref{momentum_dyn}).

\subsection{Parametric-pure feedback forms}

For the sake of clarity of exposition, let a given nonlinear single-input-single-output perturbed system be modeled as
\begin{align}\label{sys_siso_ct}
\dot x &= f_0(x) + \sum_{i=1}^{p} \theta_i f_i(x) + \left[ g_0(x) + \sum_{i=1}^{p} \theta_i g_i(x) \right] u \\
y &= h(x)
\end{align}
with the vector fields $f(x), g(x)$ being complete, the map $h(x)$ vanishes at the stationary points, the states $x \in \mathcal{X} \subset \mathbb{R}^n$, the controls $u \in \mathcal{U} \subset \mathbb{R}$, and assume the output $y \in \mathbb R$ \textcolor{black}{has a well-defined relative degree $r$, equivalently the input-output link} is linearizable via feedback \cite[Ch.~4]{Isidori}. Further, let $\theta \in \mathcal{Q} \subset \mathbb{R}^p$ be a set of unknown parameters in a compact and convex set. Whenever there exists a $\theta-$independent coordinates transformation $\phi(x) : x \mapsto \begin{pmatrix} \xi & \eta\end{pmatrix}$ such that in the new coordinates, system (\ref{sys_siso_ct}) reads:
\begin{align*}
\dot \xi_1 &= \xi_2 + \theta^\top \alpha_1(\xi_1, \xi_2, \eta) \\
\dot \xi_2 &= \xi_3 + \theta^\top \alpha_2(\xi_1, \xi_2,\xi_3, \eta) \\
& \vdots \\
\dot \xi_{r-1} &= \xi_r + \theta^\top \alpha_{r-1}(\xi_1, \dots, \xi_{r}, \eta) \\
\dot \xi_r &= \alpha_0(\xi, \eta) +  \theta^\top \alpha_r(\xi_ \eta) + \beta_0(\xi, \eta)  u \\
\eta &= q_0(\xi, \eta) + \sum_{i = 1}^{p} \theta_i q_i(\xi, \eta)\\
y &= \xi_1
\end{align*}
and such that $\alpha_i(0,0) = 0, \ \forall i = 0,1, \dots, r$ and $\beta_0(0,0) \not = 0$, then the above representation in the new coordinates is referred to as a \textit{parametric-pure feedback system} \cite{pure_parametric_feedback_form}. For systems transformable into the parametric-pure feedback form above, a systematic procedure is detailed in \cite{pure_parametric_feedback_form} for designing robust adaptive feedback controllers. 
\begin{remark}\label{centroidal_dynamics_parametric_pure_feedback}
    Note that setting $\xi_1 = p_{CoM}, \ \xi_2 = h^\ell, \ \eta = h^\omega$,  the perturbed centroidal dynamics (1) is readily transformed into a parametric-pure feedback form \textcolor{black}{with $r = 2$.}
\end{remark}
\subsection{Problem statement}
\begin{assumption}\label{assumption:1}
before stating the problem, we make the following assumption;
    \begin{enumerate}[(i)]
        \item We are given a desired \textit{nominal} reference trajectory for the center of mass, $p^n_{CoM}$ together with both the first and second-order derivatives of the reference.
        \item \Moesays{Consider in (1) that $ m = 1$ \textcolor{black}{and  $n_c = k = 1$}}.
        \item Assume that the disturbance force $\theta$ acting on the CoM is \textit{constant}
    \end{enumerate}  
\end{assumption}

\begin{remark}\label{assumptions_restrictiveness}
    point $(i)$ in Assumption (\ref{assumption:1}) is not restrictive. One could compute the derivatives if not given. The above also applies to constant references. Concerning point $(ii)$, it is simply for clarity of presentation, and the arguments following in the next section hold for the general case (\Moesays{modulo simple matrices manipulations}). \Moesays{Point (iii) restricts the class of disturbances, for which the statements made later in this paper hold, to persistent constant disturbances (e.g. forces at the robot hand while carrying a non-changing payload)}. However, as we will see in the experiments and validation reported in Section \ref{sec:validation}, for some bounded impulsive disturbances (pushes and payload weight changes), the proposed solution still performs reasonably well \textcolor{black}{thanks to the inherent step adjustment capabilities of the controller}.
\end{remark}
\begin{problem}\label{problem_statement}
Given the perturbed dynamics (1), and assuming Assumption \ref{assumption:1} holds, design a control input \(u\) such that:
\begin{enumerate}[$i$]
    \item \textcolor{black}{For a disturbance \(\theta\) that is bounded in the \(\ell_2\)-norm, the tracking error remains bounded, specifically:}
    $\lim_{t \to \infty} \|p_{CoM} - p_{CoM}^n\| \leq \epsilon
    $
    for some small \(\epsilon \in \mathbb{R}\).
    
    \item The closed-loop feedback system is \textit{Lyapunov stable}, i.e., solution trajectories of the closed-loop system remain near an equilibrium for all time.
    \footnote{\textcolor{black}{One could ask for the stronger notion of \textit{asymptotic stability}, i.e. that solution trajectories not only remain near but also converge to the equilibrium as \(t \to \infty\). This is not strictly necessary in our setting since we do \textit{not} require perfect tracking.}} \hfill $\vartriangleleft$
\end{enumerate} 
\end{problem}


\section{Online centroidal MPC with stability and robustness  guarantees}\label{Sec:three}

In what follows, we derive explicit expressions for stabilizing constraints and cost regularization terms. \textcolor{black}{These ingredients, when added to the standard Centroidal MPC formulation can address the Problem \ref{problem_statement}}. 
\subsection{Robust adaptive redesign}
Noting Remark \ref{centroidal_dynamics_parametric_pure_feedback} and the procedure reported in \cite{pure_parametric_feedback_form}, let us define the following coordinates change:
\begin{subequations}\label{coordinates_change}
\begin{align}
z_1 &= p_{CoM} -  p_{CoM}^n \\
z_2 &= k_1 (p_{CoM} -  p_{CoM}^n) + Bh - \dot{p}_{CoM}^n\\
\eta &= Ch.
\end{align}
\end{subequations}
for $k_1$ a suitable diagonal gains matrix, and $C$ selector matrix such that $\eta = Ch = h^w$. Note that the above is a valid coordinate change choice\footnote{A diffeomorphism defined \textit{almost} everywhere (in the sense of the domain of definition of the relative degree taking in eqs. (1) as output $p_{CoM}$).}. In these new coordinates, we can rewrite the $z$ sub-dynamics corresponding to (\ref{momentumdyn_ct}) as: 
\begin{align*}
\dot z_1 &= Bh-\dot{p}_{CoM}^n\\
&= B B^\top\left(-k_1 z_1 + z_2 +  \dot{p}_{CoM}^n \right) - \dot{p}_{CoM}^n \\
&= -k_1 z_1 + z_2 \\
\dot z_2&= k_1 \dot z_1 + B\dot h - \ddot{p}_{CoM}^n\\
&= k_1\left(-k_1 z_1 + z_2 \right) + B [ A_u(p_u)u  + A_\theta(p_\theta)\theta \\&+ \Vec{g} ] - \ddot{p}_{CoM}^n \\
&= -k_1^2z_1 + k_1z_2 + B\Vec{g} +  u + \theta - \ddot{p}_{CoM}^n.
\end{align*}
Now, denote by $\hat \theta$ an estimate of $\theta$ that we \textit{adaptively} update, and set, for some suitable gain matrix $k_2$ the feedback:
\begin{align}\label{adaptive_feedback}
 u &=  -(k_1 + k_2) z_2 + k_1^2z_1 - B\Vec{g} -  \hat \theta + \ddot{p}_{CoM}^n 
\end{align}
and note that the above choice in turn yields:
\begin{align*}
 \dot z_2 &= -k_2 z_2 +  \tilde \theta
\end{align*}
where $\tilde \theta$ is the estimation error, i.e. $\tilde \theta = \theta - \hat \theta$. Finally, for an adaptive law for our estimate, let
\begin{align}\label{adaptation}
\textcolor{black}{\dot{\hat \theta}} &=  \textcolor{black}{z_2}
\end{align}
With the feedback (\ref{adaptive_feedback}) and the adaptive law above, the overall dynamics becomes (\Moesays{with some abuse mixing coordinates, \textcolor{black}{i.e. using both $z, \eta$, and $p_{\text{CoM}}$ state variables on the right-hand side}}), 
\begin{subequations}
\begin{align}
\dot z_1 &= -k_1 z_1 + z_2 \\
\dot z_2 &= -k_2 z_2 +  \tilde \theta \\
\textcolor{black}{\dot {\hat \theta}} &=  \textcolor{black}{z_2} \\
\dot \eta &=  S(p_u - p_{CoM})\bigg(-k_1 z_2  - k_2 z_2 + k_1^2z_1 \nonumber \\ &- B\Vec{g} -  \hat \theta + \ddot{p}_{CoM}^n \bigg)  + S(p_\theta - p_{CoM})\theta
\end{align}
\end{subequations}
 
\begin{remark}\label{remark_integrator_more_forces}
    Note that the first $r = 2$ derivatives of the reference are needed by construction in the above argumentation. Hence this design assumes a higher-level planner for the CoM positions, velocities, and accelerations. Additionally, when $m \not = 1$ the procedure is unchanged (apart from minor computations adjustment).
\end{remark}

From the above, we can state the following intermediate and helpful result:
\begin{lemma}\label{lemma_1}
    Consider the centroidal dynamics (\ref{momentumdyn_ct}), and let Assumption \ref{assumption:1} hold true, then the feedback:\begin{align}\label{adaptive_feedback_with_extra_term}
    u &= u_n + \nu
\end{align}
with $u_n$ given by (\ref{adaptive_feedback}) and $\nu$ an additional term solving the inequality
\begin{align}\label{nu_inequality}
     \big(z_2^\top + \eta^\top &S(p_u - p_{\text{CoM}})\big) \nu \leq \nonumber \\ &-\eta^\top \big( S(p_u - p_{\text{CoM}}) u_n + S(p_\theta - p_{\text{CoM}}) \theta \big)
\end{align}
    together with the coordinates change (\ref{coordinates_change}) and the adaptation law (\ref{adaptation}) solve Problem \ref{problem_statement}.
\end{lemma}    
\begin{proof}
one may note that in the new coordinates, Problem \ref{problem_statement} reduces to asking for stability of the closed-loop system. Now pick the candidate Lyapunov function 
\begin{align}\label{Lyaponuv_function}
V(z, \tilde \theta, \eta) &= z_1^\top z_1 + z_2^\top z_2 + \tilde \theta^\top \tilde \theta + \eta^\top \eta
\end{align}
and note that for $k_1, k_2 > 0$, 
\begin{align*}
	\dot V &= z_1^\top \dot z_1 + z_2^\top \dot z_2 - \tilde \theta^\top \dot{\hat \theta} + \eta^\top \dot \eta \\
	&= z_1^\top \left( -k_1 z_1 + z_2 \right) 
	+ z_2^\top \left( -k_2 z_2 + \tilde \theta + \nu \right) 
	- \tilde \theta^\top z_2 \\
	&\quad + \eta^\top \big( S(p_u - p_{\text{CoM}})[u_n + \nu] + S(p_\theta - p_{\text{CoM}}) \theta \big) \\
	&= -z_1^\top k_1 z_1 - z_2^\top k_2 z_2 + z_1^\top z_2 
	+ z_2^\top \nu \\
	&\quad + \eta^\top \big( S(p_u - p_{\text{CoM}})[u_n + \nu] + S(p_\theta - p_{\text{CoM}}) \theta \big) \\
	&\leq z_2^\top \nu 
	+ \eta^\top \big( S(p_u - p_{\text{CoM}})[u_n + \nu] + S(p_\theta - p_{\text{CoM}}) \theta \big).
\end{align*}

because the first three terms are negative semi-definite. Then it is clear that having $\dot V \leq 0$ reduces to finding $\nu$ solving (\ref{nu_inequality}). This inequality admits solutions thanks to the skew-symmetric nature of $S(\cdot, p_{\text{CoM}})$ and $\theta$ being constant by assumption. Consequently, under (\ref{adaptive_feedback_with_extra_term}) the closed-loop system is stable, and so is the $z$ sub-dynamics, thus satisfying both the bounded tracking error and stability requirements of Problem \ref{problem_statement}.
\end{proof}

\begin{remark}\label{remark_zero_dynamics}
\textcolor{black}{note that $u_n$ stabilizes the $(z, \hat \theta)$ sub-dynamics, and the angular momentum dynamics restricted to the set $\mathcal{S} = \{ (z, \eta): z = 0 \}$ becomes an internal zero dynamics in that case. Given that the angular momentum dynamics is non-minimum phase, then $u_n$ alone is not enough to stabilize the full closed-loop system. Consequently, we leverage $u_n$ for adaptation and gravity compensation, while introducing an additional term $\nu$ solution to (\ref{nu_inequality}) that stabilizes the resulting closed-loop. In the following subsection, we will use optimization, within our MPC formulation, to obtain the value of $\nu$ enforcing a stability constraint on the system and satisfying additional constraints.}

\end{remark}

To conclude this subsection, we emphasize that
the \Moesays{impact} forces computed by (\ref{adaptive_feedback}) may not guarantee feasible robot locomotion. This is mainly due to having \textit{no} considerations for the \textit{feasibility} of the contact forces (in the sense of the friction cone constraints \cite{Featherstone}). \textcolor{black}{This limitation is addressed in the following section. In particular, we will use the knowledge gained from Lemma \ref{lemma_1} and the discussion that followed to derive a modified Centroidal MPC controller that considers friction cone constraints}. 

\subsection{Stable Centroidal MPC for robust  locomotion}\label{mpc_redesign}
Consider the following cost functional
\begin{align}\label{cost_function}
  \mathcal{J} &= \sum_{k = 0}^{n_p} T_{z_1(k)} + T_{\eta(k)} + T_{p_{\mathcal{C}}(k)} + T_{u(k)}    
\end{align}
where $n_p \geq 1$ is the \textit{prediction horizon} of the MPC controller \cite{grune}, $k \in \mathbb{Z}_{\geq 0}$ is the time step and 
\begin{align}\label{momentum_task}
    T_{z_1(k)} &= z_1(k)^\top Q_1 z_1(k)\\
    T_{\eta(k)} &= \eta(k)^\top Q_2 \eta(k)
\end{align}
for some positive definite weight matrices $Q_1, Q_2 > 0$ penalizing linear and angular momentum errors. Additionally,
\begin{align}
    T_{p_{\mathcal{C}}(k)} &= (p_{\mathcal{C}}(k) - p_{\mathcal{C}}^n(k) )^\top Q_3 (p_{\mathcal{C}}(k) - p_{\mathcal{C}}^n(k) )
\end{align}
is a task penalizing the deviation of the \textit{feet contact locations} $p_{\mathcal{C}}$ which has the dynamics described by eq. (5) in \cite{GiulioICRA} from a nominal contact location $p_{\mathcal{C}}^n(k)$ at time instant $k$. Finally, 
$T_{u(k)}$ is a regularization task on the control \Moesays{impact} forces ( making forces on the feet corners as symmetric as possible (cf term (Eq. (6) in \cite{GiulioICRA})). 

With this cost function, we associate the following constraints for all $k = 0, \dots, n_p$;
\begin{enumerate}
    \item \Moesays{The prediction model: namely a Forward-Euler integration of the unperturbed dynamics i.e. };
    \begin{equation}\label{dyn_constr}
    \begin{aligned}
        p_{CoM}(k+1) &= p_{CoM}(k) + \Delta Bh \\
 h(k+1) &= \textcolor{black}{h(k) + \Delta ( \sum_i^{n_c} A_{u_i}(p_{u_i})\Gamma_iu_i(k) +  m\Vec{g} )} \\
 p_{u_i}(k+1) &= p_{u_i}(k) + \Delta ([1-\Gamma_i]v_{u_i}) 
    \end{aligned}
    \end{equation}
    where $\Delta$ is the controller sampling rate, and $p_{u_i}, \ v_{u_i}, \ \Gamma_i$ are the position, velocity and status of the contact at the corner $i$ of the robot feet. 
    \item The coordinates change (\ref{coordinates_change}) and feedback relation over the prediction horizon
    \begin{equation} \label{coordinates_change_cstr}
    \begin{aligned}
        z_1(k) &= p_{CoM}(k) -  p_{CoM}^n(k) \\
z_2(k) &= k_1 z_1(k) + Bh(k) - \dot{p}_{CoM}^n(k)\\
\eta(k) &= Ch(k) \\
    \textcolor{black}{u(k)} &= \textcolor{black}{u_n(k) + \nu(k)}
    \end{aligned}
    \end{equation}
    \textcolor{black}{with $u_n(k)$ being the feedback (\ref{adaptive_feedback}) at time $k$ and $\nu$ a decision variable}.
    \item \Moesays{The stability constraints:} \textcolor{black}{instead of enforcing the inequality (\ref{nu_inequality}) due to having no measurement of the actual disturbance $\theta$, we impose the following equivalent two constraints. In particular, the first constraint asks for the stability of the $(z, \hat \theta)$ subdynamics given the feedback (\ref{adaptive_feedback_with_extra_term}). The second complements the first by requiring the internal $\eta$ dynamics to be stable}
    \begin{equation}\label{stability_cstr}
    \begin{aligned}
        -z_1^\top(k) k_1 z_1(k) &- z_2^\top(k) k_2 z_2(k) \\&+ z_1^\top(k) z_2(k) + z_2(k)^\top \nu(k) < 0 \\
        \| \eta (k+1) \| &\leq \| \eta(k)\|
    \end{aligned}
    \end{equation}

    \item The contact forces feasibility constraints:
    \begin{equation}\label{friction_cone}
    \begin{aligned}
        A R_{\mathcal{C}}^\top u(k) \leq b
    \end{aligned}
    \end{equation}
    where $R_{\mathcal{C}}$ is the rotation matrix associated with the impact force w.r.t the inertial frame, and $A, \ b$ constants depending on the friction coefficient \cite{Featherstone}. \textcolor{black}{One stresses that the friction cone constraints are applied to the whole feedback (\ref{adaptive_feedback_with_extra_term}) and not only on the term $\nu$.}
    \item Constraint on the maximum allowable contact location adaptation error
    \begin{align} \label{contact_location_constraint}
        \ell_b \leq R_{\mathcal{C}}^\top (p_{\mathcal{C}}(k) - p_{\mathcal{C}}^n(k) ) \leq u_b
    \end{align}
    with $\ell_b, \ u_b$ being lower and upper-bounds.
\end{enumerate}

With the above discussion in mind, we can now make the following claim,
\begin{proposition}\label{main_claim}
    \textcolor{black}{Given system (\ref{momentumdyn_ct}), and let Assumption \ref{assumption:1} hold true. Then the feedback control (\ref{adaptive_feedback_with_extra_term}) with $\nu$ solution to the MPC problem with a cost functional (\ref{cost_function}), and constraints (\ref{dyn_constr}), (\ref{coordinates_change_cstr}), (\ref{stability_cstr}), (\ref{friction_cone}), (\ref{contact_location_constraint})
    solves Problem \ref{problem_statement} whenever the optimal control problem is recursively feasible.}
\end{proposition}

\Moesays{The above statement follows directly from Lemma \ref{lemma_1}. Due to (\ref{adaptive_feedback_with_extra_term}) and its accompanying discussion, one in principle solve for $\nu$, treating $u_n$ as a feed-forward control while ensuring the friction cone constraints apply to their sum. Note that in the above statement, we neglect the effects of discretization on MPC problems \cite{ElobaidNolcos}. Instead, we implicitly assume that emulation (Zero-Order-Holding) of control and measurement signals are fast enough such that we are able to use the Lyapunov arguments in continuous-time. This is not restrictive in practice and is typically the case when dealing with robotics applications.}

\Moesays{One could interpret the above reformulation as follows: The MPC problem above, in essence, is the \enquote{projection} of the feedback (\ref{adaptive_feedback_with_extra_term}) in a set defined by the force-feasibility and maximum contact adaptation errors constraints\footnote{in essence, a set being the union of two convex polytopes defined by the Centroidal MPC constraints.}. This also explains the choice of the candidate Lyapunov function for which we have intuition about the \textit{existence} of a feedback solving the unconstrained problem. Additionally, from an implementation point-of-view, since one invokes receding horizon by design, it is possible to relax the problem and enforce the constraints (\ref{coordinates_change_cstr})-(\ref{stability_cstr}) only over the first predicted value, and not necessarily over the whole horizon. The intuition is that only the first optimal control in the sequence is applied before the optimizer is recomputed.  This further enhances the computational time required \cite{clf_nmpc}.} 
\Moesays{\begin{remark}\label{comments_on_proof}
        with respect to \cite{ElobaidICRA}, it is worth stressing that the MPC problem defined in Proposition \ref{main_claim} does not consider the disturbance in the prediction model. Furthermore, no specific tasks handling an assumed persistent disturbance is present. 
\end{remark}}


\section{Validation and experiments}\label{sec:validation}

\Moesays{Hereinafter, we first present an answer to the question \enquote{why we care about stability?}, then proceed to discuss experimental validation on both the humanoid ergoCub and the quadruped Aliengo\footnote{Available in \url{ https://www.unitree.com}} \textcolor{black}{together with some statistics.}}
\subsection{Do we need the stabilizing constraints ?}\label{edge_cases}

\Moesays{Apart from rigor, it turns out that the stabilizing constraints allow for some \textit{practical} cases where not including them leads to close loop instability, and consequently failure of the robot to complete a locomotion task. More precisely, consider a simple unperturbed floating mass system equipped with legs, with the CoM height being $\SI{0.53}{m}$. This system is asked to follow a given nominal reference for the CoM and feet contact locations. To make the comparison fair, no disturbances are applied to the system throughout the whole trajectory. We use the Centroidal MPC with and without the stabilizing constraints to assess their effect. }

\Moesays{Figure\ref{fig:stability_constraints_effects} reports the results of this comparison on a floating mass with two legs resembling a humanoid. At first, the prediction horizon is set to be around $\SI{1.2}{s}$ similar to what is done in \cite{GiulioICRA}, while the controller frequency is set to $\SI{10}{Hz}$. We allow $\ell_b, \ \ell_u \not = 0$ to add step adjustment capability and help the optimizer. Both the nominal Centroidal MPC and our proposed reformulation are able to track the desired nominal references. We then decrease the prediction horizon to $\SI{0.9}{s}$ at which point the nominal formulation fails to perform a single step, while the proposed reformulation, thanks to the stabilizing constraints, keeps both the CoM and angular momentum trajectories bounded, thus completing the locomotion task. A similar effect can be observed by halving the controller frequency, i.e. setting $\Delta T = \SI{0.2}{s}$. }
\Moesays{\begin{remark}\label{feasibility_discussion}
    the above discussion also highlights an interesting observation: while the computational time required by the proposed reformulation is higher compared to the original work in \cite{GiulioICRA} (as will be seen when discussing experiments), this can be mitigated, at least in part, by shortening the prediction horizon. The only limitation in this case is \textit{recursive feasibility} of the problem,  explicitly assumed in Proposition \ref{main_claim}.  
\end{remark}}

\begin{figure*}[!t]
    \centering
    \setlength{\columnsep}{0.02cm} 
    \begin{multicols}{3}
        \begin{tikzpicture}[scale=0.95] 
            \node[inner sep=0pt] (image1) at (0,0) {\includegraphics[width=0.9\linewidth]{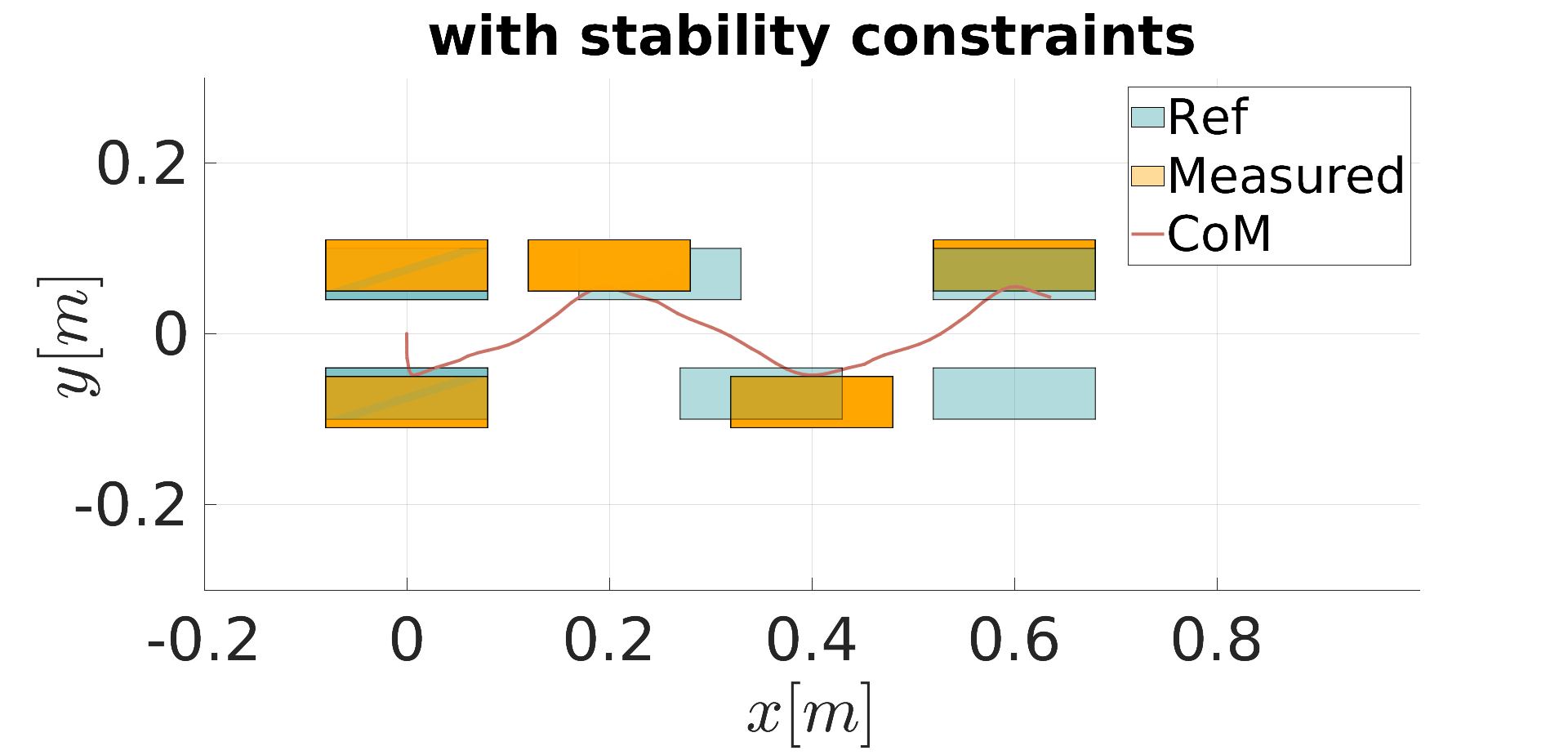}};
            \node[inner sep=0pt] (image2) at (0,-2.5) {\includegraphics[width=0.9\linewidth]{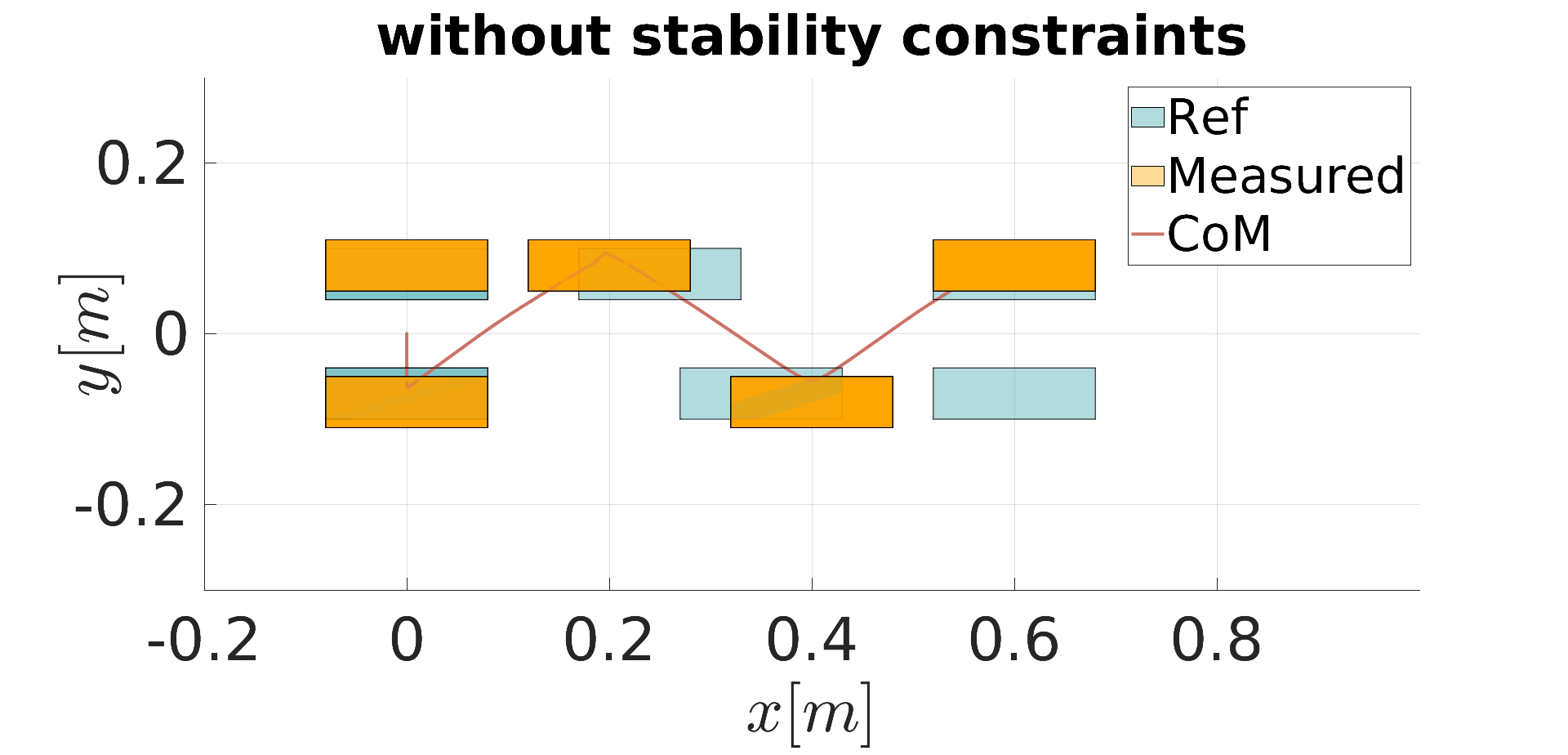}};
            \draw[line width=0.8pt, rounded corners] ([shift={(-0.2,0.2)}]image1.north west) rectangle ([shift={(0.2,-0.2)}]image2.south east);
            \node[align=center, above, yshift=0.2cm] at (image1.north) {\scriptsize $n_p = 12, \ \Delta T = 0.1$};
        \end{tikzpicture}

        \begin{tikzpicture}[scale=0.95] 
            \node[inner sep=0pt] (image1) at (0,0) {\includegraphics[width=0.9\linewidth]{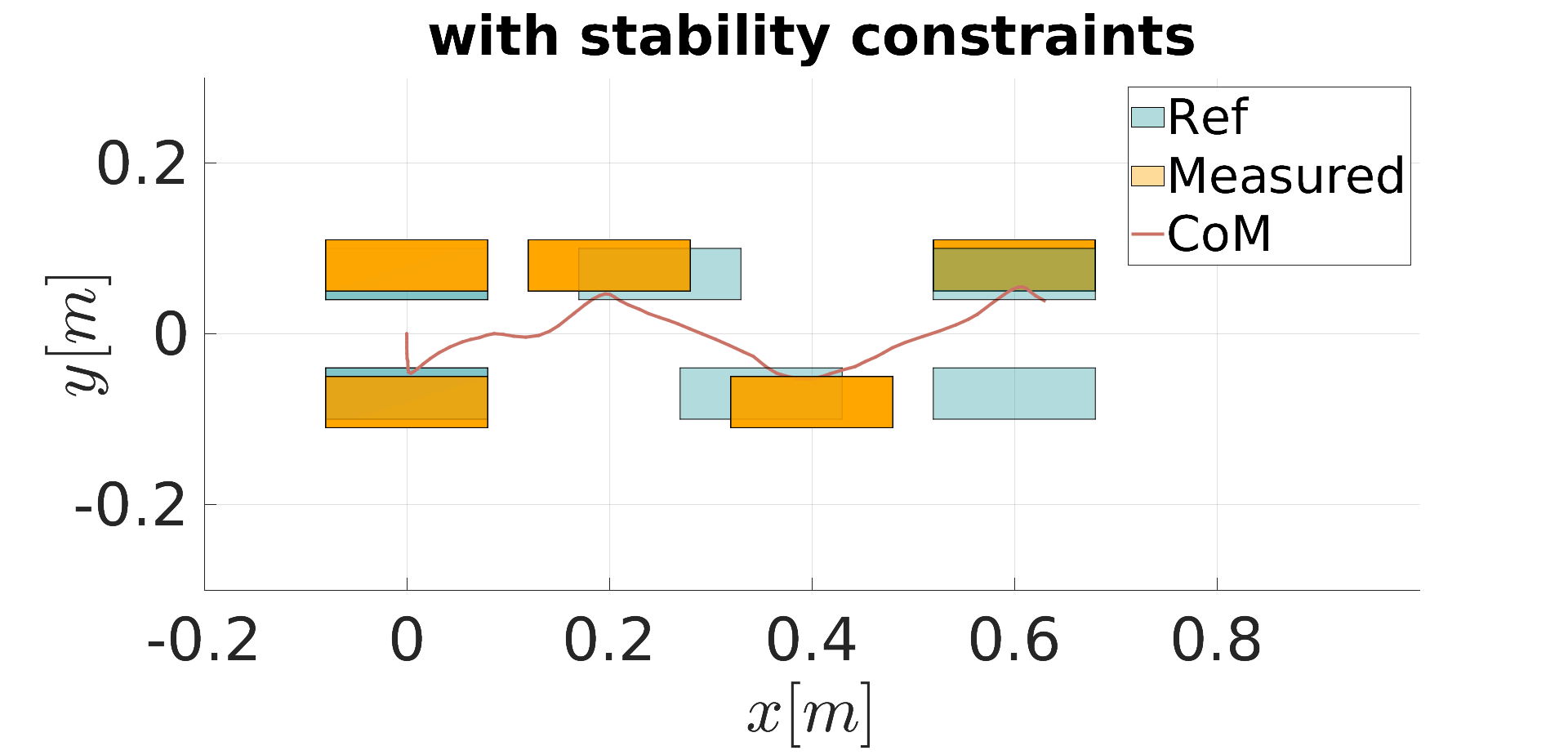}};
            \node[inner sep=0pt] (image2) at (0,-2.5) {\includegraphics[width=0.9\linewidth]{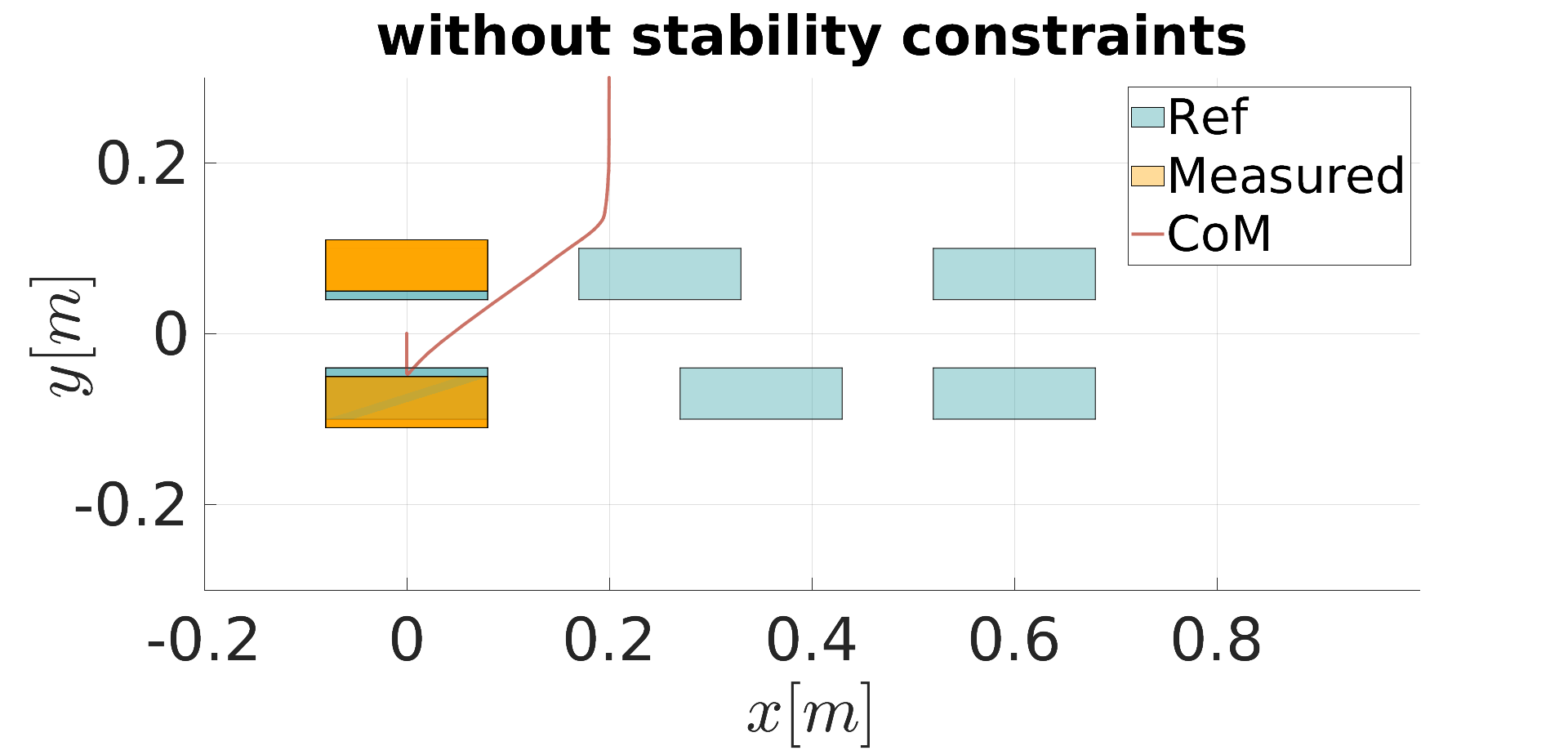}};
            \draw[line width=0.8pt, rounded corners] ([shift={(-0.2,0.2)}]image1.north west) rectangle ([shift={(0.2,-0.2)}]image2.south east);
            \node[align=center, above, yshift=0.2cm] at (image1.north) {\scriptsize $n_p = 10, \ \Delta T = 0.1$};
        \end{tikzpicture}

        \begin{tikzpicture}[scale=0.95] 
            \node[inner sep=0pt] (image1) at (0,0) {\includegraphics[width=0.9\linewidth]{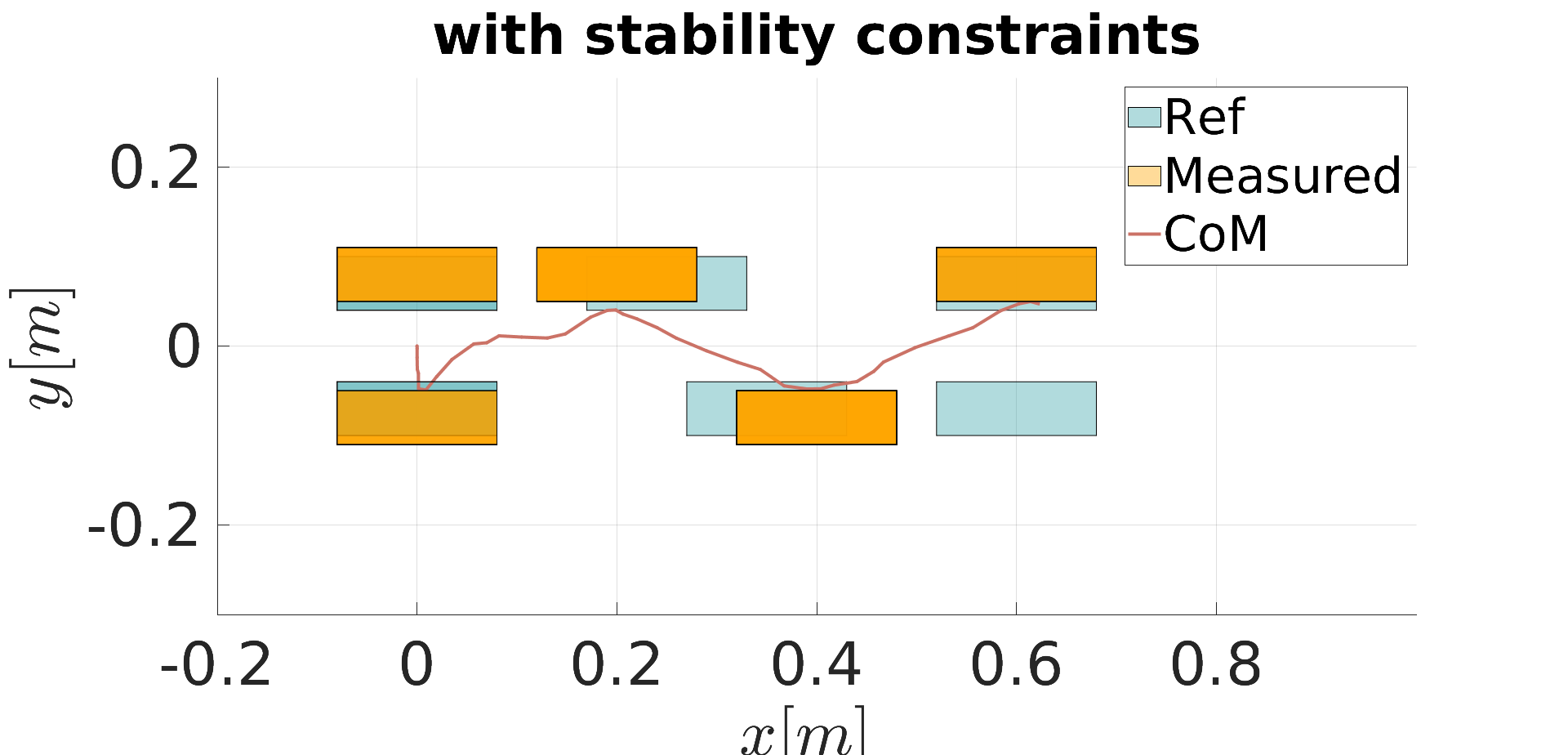}};
            \node[inner sep=0pt] (image2) at (0,-2.5) {\includegraphics[width=0.9\linewidth]{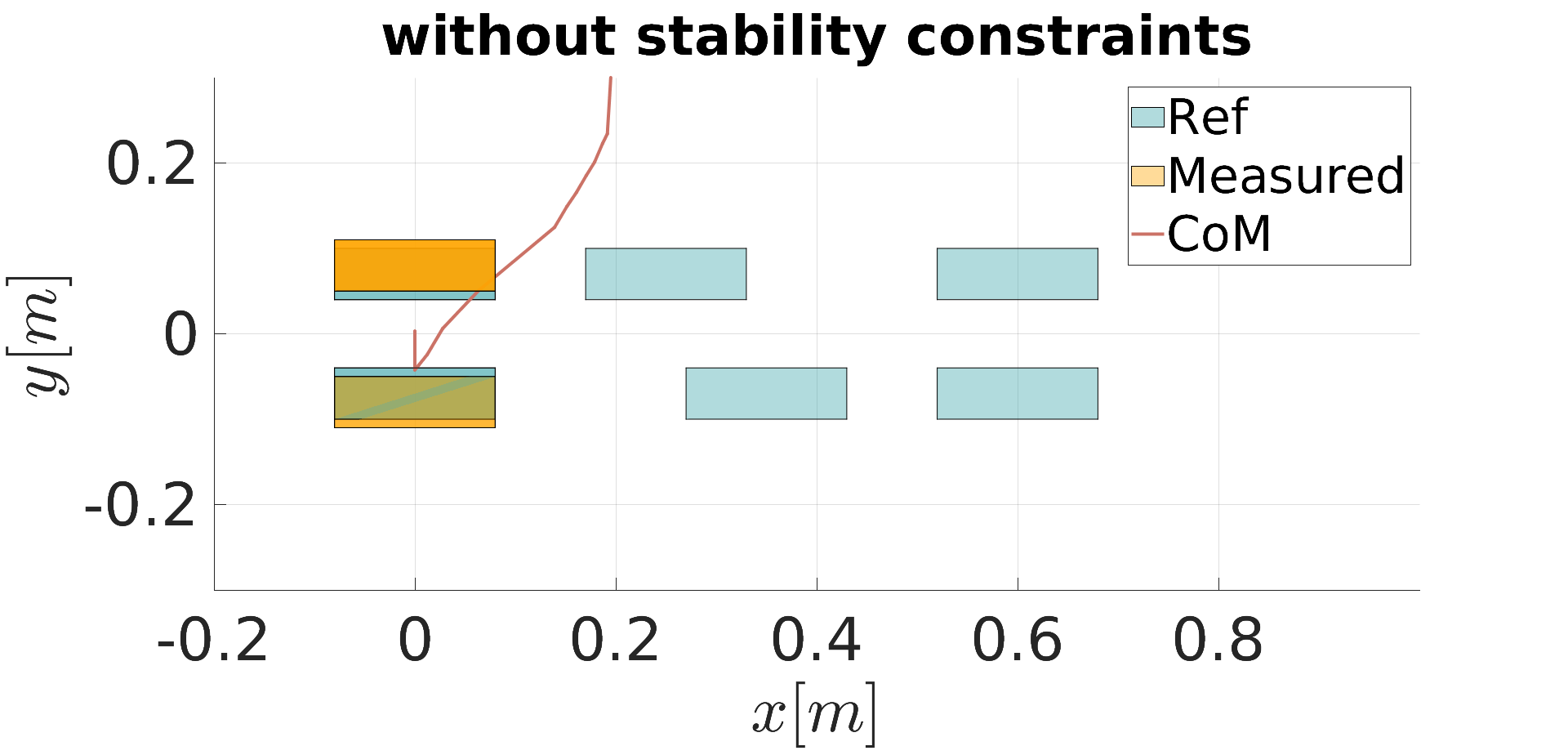}};
            \draw[line width=0.8pt, rounded corners] ([shift={(-0.2,0.2)}]image1.north west) rectangle ([shift={(0.2,-0.2)}]image2.south east);
            \node[align=center, above, yshift=0.2cm] at (image1.north) {\scriptsize $n_p = 12, \ \Delta T = 0.2$};
        \end{tikzpicture}
    \end{multicols}
    \caption{Left - nominal prediction horizon and controller frequency. Center - shorter horizon and nominal frequency. Right - nominal horizon and lower frequency. In the last two cases, the proposed method succeeds in stabilizing the robot's motion as opposed to the nominal one thanks to the additional constraints (\ref{stability_cstr}).}
    \label{fig:stability_constraints_effects}
\end{figure*}

\subsection{The case of a humanoid robot}

The newly developed ergoCub\footnote{\url{https://ergocub.eu/project}} robot, designed and built by the Italian Institute of Technology to be a successor to the iCub 3.0 robot \cite{icub3_avatar},  weighs almost the same at $\SI{56.7}{kg}$ while standing $\SI{25}{cm}$ taller. 
To validate our proposed approach, the optimization problem presented in Proposition \ref{main_claim} is implemented (using casadi \cite{casadi} with IPOPT \cite{ipopt}, running at \textcolor{black}{$\SI{20}{Hz}$}) as the middle layer in the three-layered control architecture presented in \cite[Figure~2]{GiulioICRA}. The output of this controller is the desired contact forces and locations. These in turn are passed to a whole-body control layer that generates and sends desired joint positions to the low-level robot motor control boards. The whole-body controller is a stack of task QP (solved using an off-the-shelf solver running at \textcolor{black}{$\SI{200}{Hz}$}) over the robot generalized velocity, which are then integrated as position references for the low-level control.

To facilitate computations, the stability constraints and changes in coordinates are enforced only for the first step of the optimization horizon. Additionally, the contraction constraint on $\eta(k)$ is reduced to a limit on the norm, i.e.$\|\eta(k) \| \leq \bar \alpha$ for some $\bar \alpha \in \mathbb R$. Furthermore, the gain matrices \textcolor{black}{$k_1 = 0.1\, I_3, \ k_2 = 0.5\, I_3$} are fixed, thus reducing the optimization variables and removing them from the constraints. 

Several experiments were carried out in which the robot \textcolor{black}{is asked to\footnote{see accompanying video for the testing scenarios.\label{accompanying_video_footnote}}: (1) carry a box with fixed payload weight and complete a locomotion task (with two different speeds, one slower by a factor of $0.7$ compared to the other), (2) carry a varying payload (this is achieved by a person putting different weight plates inside the box up to $\SI{6}{kg}$)  while the person is exerting external pushes, and, as a benchmark against the inherent step adjustment capability of the MPC (3) walk while the person is exerting external pushes on the robot estimated as $\SI{60}{N} - \SI{100}{N}$}. In Figure. \ref{fig:performances}, we report the results of the second case where towards the end the robot is subjected to a large lateral pull by the human on the right forearm. The following observations can be made
\begin{itemize}
    \item To cope with the payload, the tracking error on the nominal contact is never zero. The robot shifts the right foot contact slightly to adapt to the changing payload, i.e. shifting and increasing weight inside the box that the robot is carrying.
    \item The controller is able to keep the internal dynamics (angular momentum modulo robot mass) norm below $\bar{\alpha} = 0.3$, even during the pulling phase (see Figure \ref{fig:performances}).
\end{itemize}

It is also worth mentioning that the feet contact position tracking error is comparable to that reported in \cite{ElobaidICRA}, \textcolor{black}{being at most $\SI{0.055}{m}$ on the second experiment scenario}, despite the larger payload (in comparison), additional constraints and different footprint dimensions. 

\begin{figure*}[!t]

\tcbset{
    colframe=black, colback=white, boxrule=0.5mm, width=\linewidth,
    title=Proposed MPC performance on the humanoid ergocub, 
    fonttitle=\bfseries,
    coltitle=white,    
    colbacktitle=gray 
}

\begin{tcolorbox}
    \centering
    \includegraphics[width=0.8\textwidth]{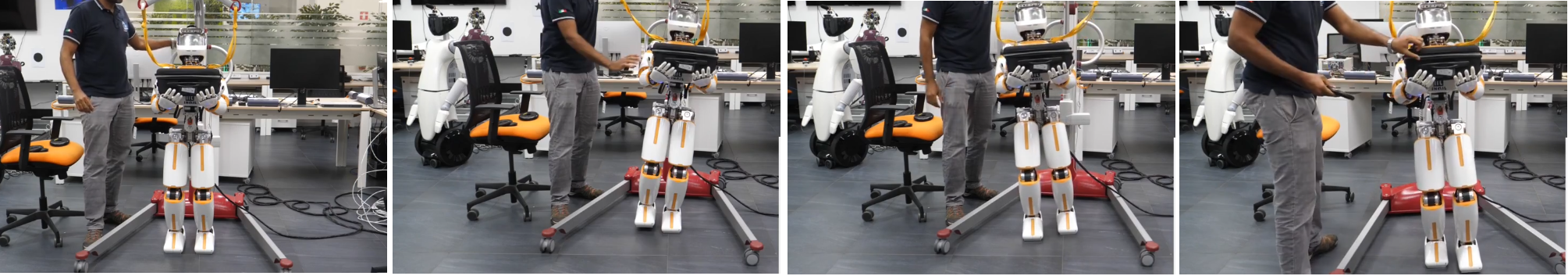}
    \begin{multicols}{2}
        \includegraphics[width=0.8\linewidth]{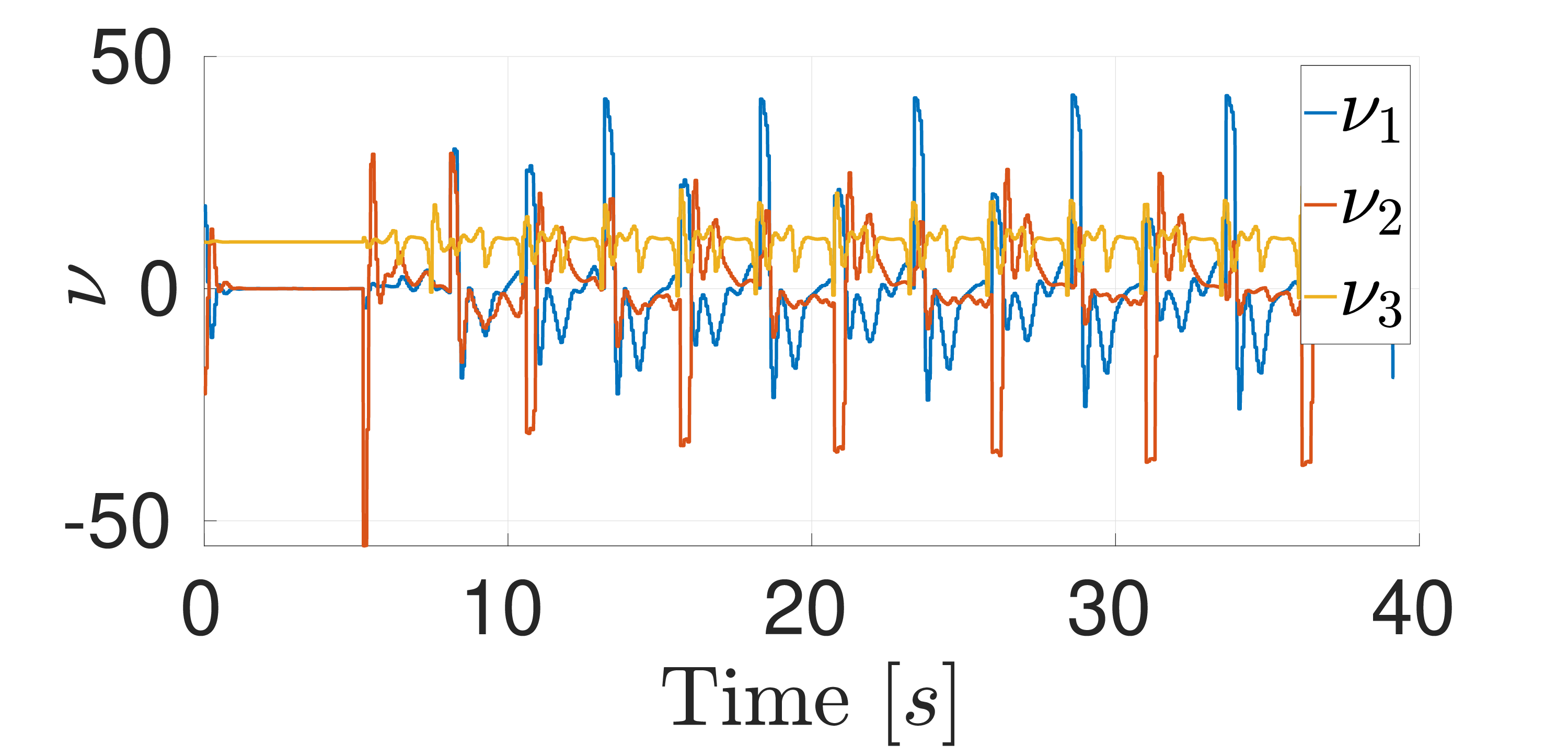}
        \includegraphics[width=0.8\linewidth]{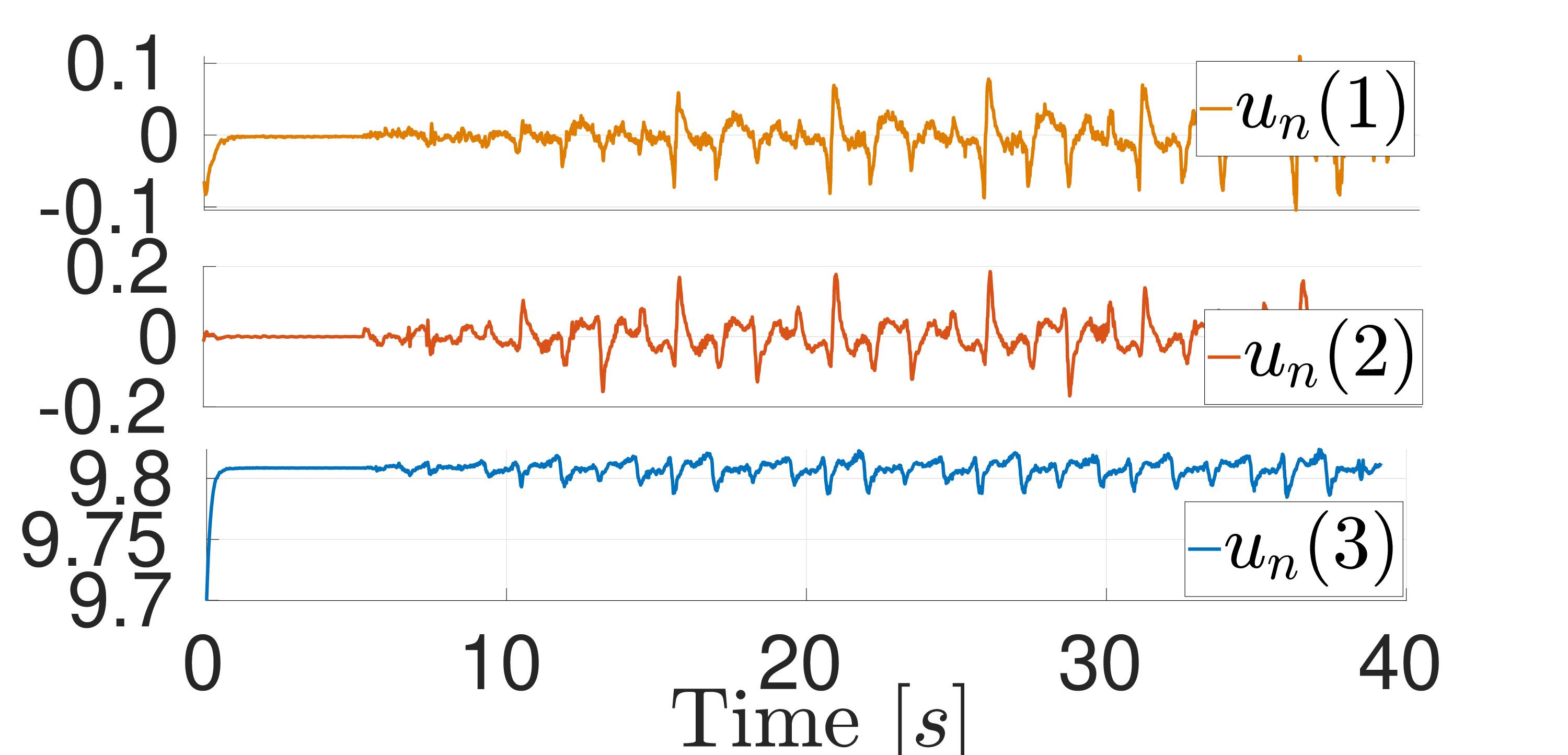}
    \end{multicols}
\end{tcolorbox}

\tcbset{
    colframe=black, colback=white, boxrule=0.5mm, width=\linewidth,
    title=Nominal Vs Proposed MPC on the quadruped Aliengo,
    fonttitle=\bfseries,
    coltitle=white,    
    colbacktitle=gray 
}

\begin{tcolorbox}
\centering
\includegraphics[width=0.8\textwidth]{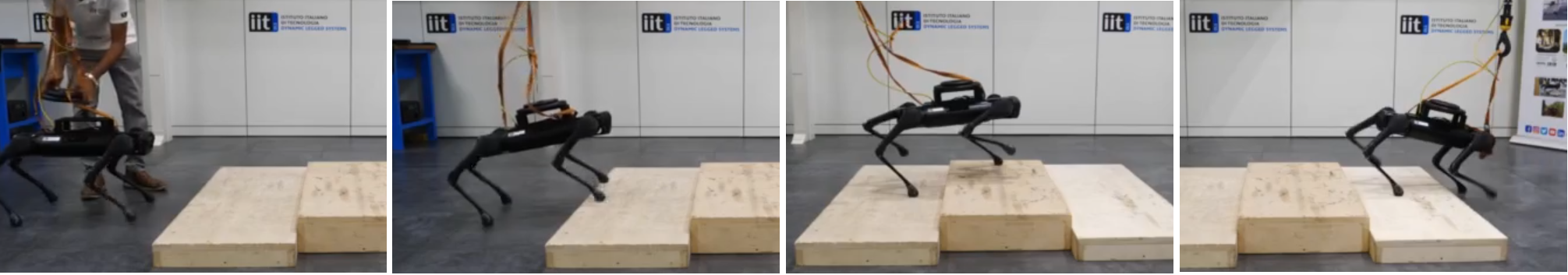}
    \begin{multicols}{2}
        \includegraphics[width=0.8\linewidth]{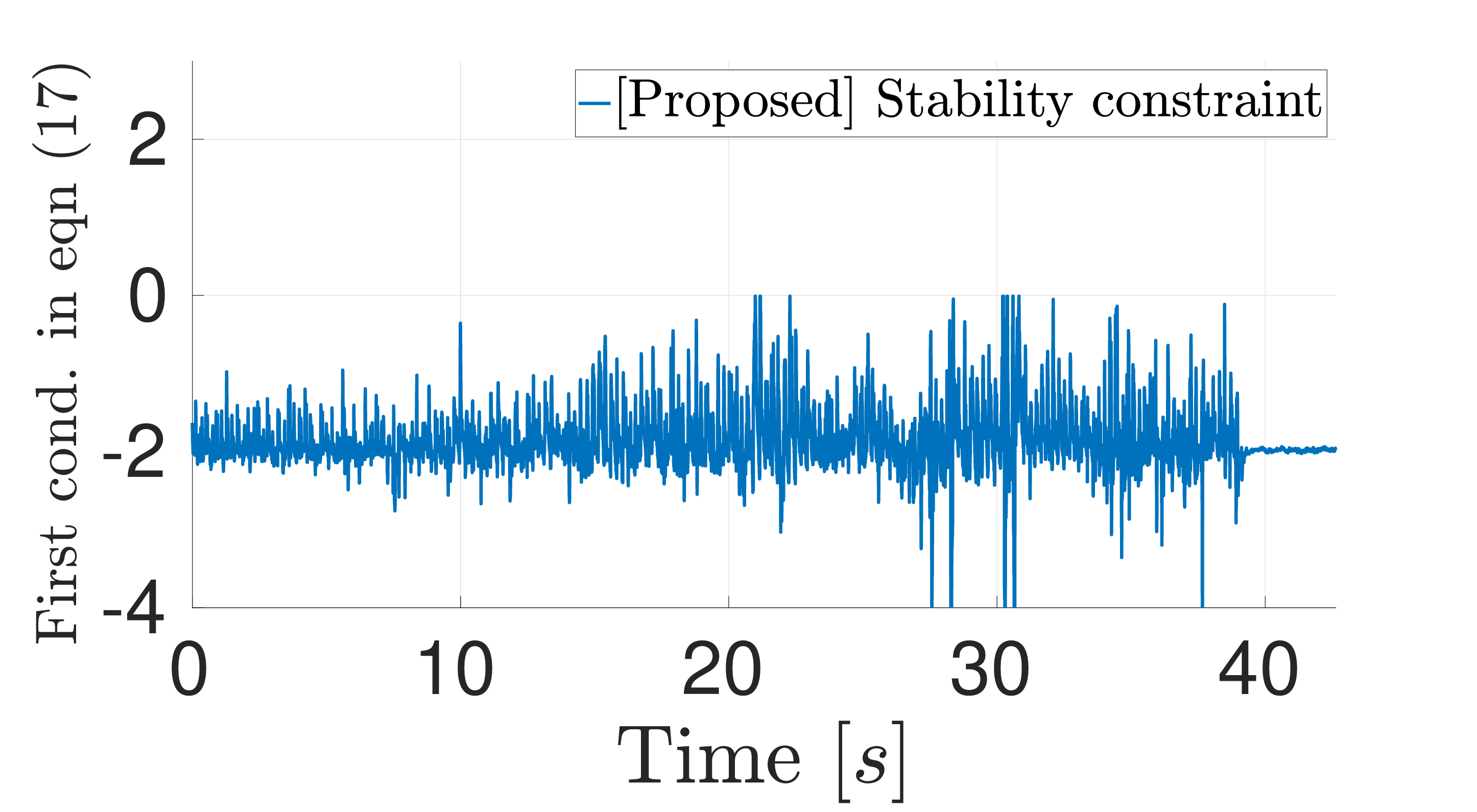}
        \includegraphics[width=0.8\linewidth]{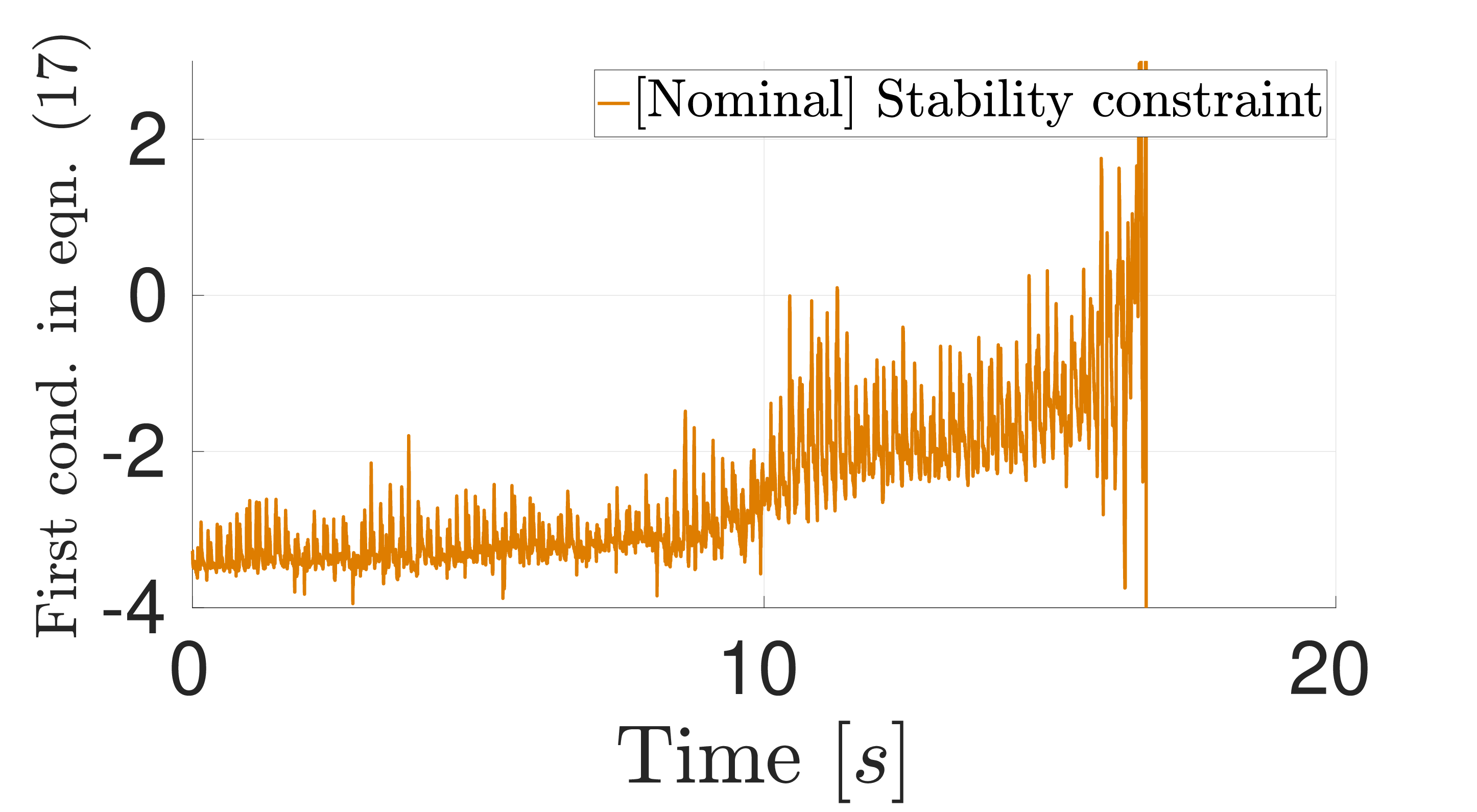}
    \end{multicols}
\end{tcolorbox}

\caption{Top left: \textcolor{black}{The nominal feedback (5) (scaled by robot the mass for clarity) for the humanoid experiment. Top right: the extra stabilizing feedback term $\nu$. Bottom left: the proposed MPC satisfying the first stabilizing constraint in (17). Bottom right: the nominal MPC violating the stabilizing constraint when the quadruped fails (see accompanying media)}. }
\label{fig:performances}

\end{figure*}

\begin{table*}[h!]
	\centering
	\renewcommand{\arraystretch}{1.65} 
	\setlength{\tabcolsep}{8pt}      
		\resizebox{\textwidth}{!}{
		\begin{tabular}{|>{\columncolor[gray]{0.9}}c|c|c|c|c|c|c|}
			\hline
			\rowcolor[gray]{0.9}
			\multirow{2}{*}{\textbf{Metric}} & \multicolumn{2}{c|}{\cellcolor[gray]{0.9}\textbf{0 kg}} & \multicolumn{2}{c|}{\cellcolor[gray]{0.9}\textbf{10 kg}} & \multicolumn{2}{c|}{\cellcolor[gray]{0.9}\textbf{15 kg}} \\ \cline{2-7}
			\rowcolor[gray]{0.9}
			& \textbf{Nominal} & \textbf{Proposed} & \textbf{Nominal} & \textbf{Proposed} & \textbf{Nominal} & \textbf{Proposed} \\ \hline
			\cellcolor[gray]{0.9}\textbf{Success Rate on flat terrain (\%)} & 
			\textcolor{black}{100.0} & \textcolor{black}{100.0} & 
			\textcolor{black}{100.0} & \textcolor{black}{100.0} & 
			\textcolor{black}{64.0} & \textcolor{black}{\textbf{94.0}} \\ \hline
			\cellcolor[gray]{0.9}\textbf{Tracking error on CoM Height on flat terrain (Mean ± Std) [$m$]} & 
			\textcolor{black}{$0.005 \pm 0.004$} & \textcolor{black}{$\bf{0.002 \pm 0.004}$} & 
			\textcolor{black}{$0.043 \pm 0.004$} & \textcolor{black}{$\bf{0.010 \pm 0.007}$} & 
			\textcolor{black}{$0.112 \pm 0.015$} & \textcolor{black}{$\bf{0.017 \pm 0.01}$} \\ \hline
			\cellcolor[gray]{0.9}\textbf{Success Rate on Uneven Terrain (\%)} & 
			\textcolor{black}{100.0} & \textcolor{black}{100.0} & 
			\textcolor{black}{92.0} & \textcolor{black}{\textbf{100.0}} & 
			\textcolor{black}{16.0} & \textcolor{black}{\textbf{88.0}} \\ \hline
			\cellcolor[gray]{0.9}\textbf{Tracking error on CoM Height on Uneven Terrain (Mean ± Std) [$m$]} & 
			\textcolor{black}{$0.009 \pm 0.009$} & \textcolor{black}{$\bf{0.006 \pm 0.008}$} & 
			\textcolor{black}{$0.044 \pm 0.010$} & \textcolor{black}{$\bf{0.011 \pm 0.008}$} & 
			\textcolor{black}{$0.12 \pm 0.025$} & \textcolor{black}{$\bf{0.019 \pm 0.014}$} \\ \hline
		\end{tabular}}
	\caption{Performance comparison over flat and randomized uneven terrains, with different command velocities and payloads.}
	\label{tab:performance_metrics}
\end{table*}

\subsection{The case of a quadruped robot}
To prove the generality of our approach, we test the optimization problem presented in Proposition \ref{main_claim} on a quadruped robot walking on uneven terrain. The robot is commanded to follow a velocity reference through a set of steps while carrying an unmodelled payload of $\SI{7}{Kg}$. The controller is implemented using acados \cite{Verschueren2019}. We described the robot with the centroidal model in \cite{turrisi2024sampling}, considering the robot base orientation with Euler angles (roll, pitch, and yaw) and the base angular velocity using Euler rates. Thanks to the fast solvers, our MPC was able to run at a frequency of $\SI{160}{Hz}$, eliminating the necessity of applying the stability constraints in (\ref{stability_cstr}) only in the first time step. Given the explicit parametrization of the base angles, the residual dynamics constraint resolved in a limit of the maximum allowed base angles variation and velocity. Hence, with $\eta$ in this case, we refer to the vector comprising the roll, pitch, and Euler rates, discarding the yaw angle due to the drift in its estimation. Finally, a whole-body controller running at $\SI{250}{Hz}$ was responsible for commanding joint torques to the robot.

We perform a comparison between our formulation and a nominal centroidal MPC (running at $\SI{230}{Hz}$) on the above scenario. In Figure \ref{fig:performances}, we report the results of these experiments. Given the additional payload, the robot with the nominal MPC was unable to cross the scenario, showing a large error both in height and in the base angles that hindered the stability of the motion on uneven terrain (see the accompanying video). \textcolor{black}{This behavior can be observed by looking at Figure \ref{fig:performances} - bottom-right, where we depict the violation of the first imposed stability constraint (17) during motion. Instead, with our formulation, both quantities remain limited, allowing the robot to complete the task successfully (see Figure \ref{fig:performances} - bottom-left)}.

\subsection{Comments on repeatability and performance}
\textcolor{black}{Table \ref{tab:performance_metrics} reports a statistical analysis of the proposed method's performance against its nominal counterpart. For this, we ask the robot to navigate two different scenarios using a randomized simulation: a simple flat terrain and a pyramid of stairs (see the accompanying video for illustration) while carrying different payloads. We perform 50 trials for every single case, commanding the robot's different forward (randomly chosen between $\SI{0.1}{m/s} - \SI{0.3}{m/s}$) and angular velocities (randomly chosen between $\SI{-0.2}{rad/s} - \SI{0.2}{rad/s}$). Furthermore, in the case of the pyramid stairs, after each trial, we randomize the environment parameters (\textit{rise} from 5 to 10 cm, and \textit{go} from 50 cm to 100 cm).  We note the high success rate, here described as the absence of body collision with respect to the ground, of the proposed method under all conditions compared to its counterpart. In addition, even for cases where the nominal MPC is reasonably successful (payloads $\SI{10}{kg}$), the tracking performances are lacking. The best obtained results in each scenario are highlighted in bold-face.}

\section{Conclusion}\label{sec:conclusion}
In this paper, we endow the recently proposed Centroidal MPC locomotion controller with theoretical guarantees for stability and robustness to bounded disturbances in the form of an external force acting on the robot's Center of Mass.
The presented results are then verified via several experiments on both a humanoid and a quadruped. What remains is to enlarge the class
of disturbances that can be handled by this controller and provide explicit expressions for their bounds. Additionally, relaxing the requirement of recursive feasibility is the subject of future work. \textcolor{black}{It is important to note that our proposed method suffers from the usual limitation of assuming dynamic feasibility of the generated torques and/or velocities. These limits, while currently handled at the whole-body control layer, should be addressed at the trajectory adjustment layer instead and may require the adoption of the more complex full dynamics model of the robot inside the MPC loop.}
\section*{Acknowledgment}
This work was partially supported by the Italian National Institute for Insurance against Accidents at Work (INAIL) ergoCub Project, and Honda Research and Development Japan, through a joint-lab research initiative.


\balance
\bibliographystyle{IEEEtran}      
\bibliography{biblio}                  

\end{document}